\documentclass[letterpaper, 10 pt, conference]{ieeeconf}
\IEEEoverridecommandlockouts
\usepackage{amsmath}
\usepackage{amssymb}
\usepackage{xfrac}
\usepackage{dsfont}
\usepackage{graphicx}
\usepackage[dvipsnames]{xcolor}
\usepackage{booktabs}
\usepackage[font=small,labelfont=bf]{caption}
\usepackage{macros}
\usepackage{cite}
\usepackage{url}
\usepackage{tikz}
\usetikzlibrary{arrows,arrows.meta,backgrounds,decorations,decorations.pathmorphing,positioning,fit,automata,shapes,patterns,plotmarks,calc}
\usepackage[ruled,vlined,linesnumbered]{algorithm2e}
\usepackage{newtxtext}
\usepackage{newtxmath}
\usepackage{bm}
\usepackage[small]{subfigure}
\makeatletter
\let\NAT@parse\undefined
\makeatother
\usepackage[colorlinks=false, hidelinks, backref=false]{hyperref}
\graphicspath{{figs/}}
\title{\LARGE \bf
Ternary Logic Encodings of Temporal Behavior Trees\\with Application to Control Synthesis
}
\author{Ryan Matheu, John S. Baras and Calin Belta%
\thanks{The authors are with the Institute for Systems Research, University of Maryland, College Park, USA. Emails: \{\texttt{rmatheu}, \texttt{baras}, \texttt{calin}\}\texttt{@umd.edu}.}%
}
\definecolor{goldencheese}{HTML}{F8D3A0}
\definecolor{exblue}{HTML}{2E5F7F}
\definecolor{exred}{HTML}{DA364B}
\definecolor{exgreen}{HTML}{248C2F}
\begin{document}
\maketitle
\thispagestyle{empty}
\pagestyle{empty}
\begin{abstract}
Behavior Trees (BTs) provide designers with an intuitive graphical interface to construct long-horizon plans for autonomous systems. To ensure their correctness and safety, rigorous formal models and verification techniques are essential. Temporal BTs (TBTs) offer a promising approach by leveraging existing temporal logic formalisms to specify and verify the executions of BTs. However, this analysis is currently limited to offline post hoc analysis and trace repair. In this paper, we reformulate TBTs using a ternary-valued Signal Temporal Logic (STL) amenable to control synthesis. Ternary logic introduces a third truth value \textit{Unknown}, formally capturing cases where a trajectory has neither fully satisfied nor violated a specification. We propose mixed-integer linear encodings for partial trajectory STL and TBTs over ternary logic, allowing for correct-by-construction control strategies for linear dynamical systems via mixed-integer optimization. We demonstrate the utility of our framework by solving optimal control problems.
\end{abstract}
\section{Introduction\label{sec:introduction}}
    Task modeling for autonomous systems is a challenging art with a fine line between expressivity and interpretability. Temporal logics like Signal Temporal Logic (STL) have become prominent modeling formalisms for autonomous systems thanks to their rich expressivity, ability to specify time-dependent behaviors, and real-valued robustness semantics~\cite{maler2004monitoring,fainekos2009robustness,donze2010robust}. However, STL-based task planning can quickly become intractable when modeling long-horizon temporal behaviors and fallback mechanisms, often leading designers to implement a hierarchical approach~\cite{lin2025optimization,liu2025value}.

    Behavior Trees (BTs)~\cite{colledanchise2018behavior} are an emerging control paradigm that provides designers with an intuitive graphical interface for designing high-level control strategies for autonomous systems. BTs allow for sequential execution of tasks, fallback mechanisms for dealing with failures and disturbances, and parallel processing of multiple sub-strategies (see Fig. \ref{fig:TBT} for an example). An elegant feature of BTs is their modularity, enabling designers to reuse subtrees and make changes to strategies without universal changes to the system behavior. Along with their increased usage in control and planning, interest in the formalization of BTs for the purposes of verification, correct-by-construction synthesis, and temporal logic monitoring~\cite{schirmer2024temporal} has steadily increased.

    This work builds on the notion of using BTs as temporal logic task specifications and presents mixed-integer encodings for the BT temporal operators over a ternary-valued logic. Mixed-integer encodings allow for solving discrete-time optimal control problems with mathematical programming~\cite{raman2014model}, or feasibility checking for problems that are otherwise undecidable~\cite{ghosh2016diagnosis}.

    Our primary contributions in this work are:
    \begin{enumerate}
        \item
            We concretize the formalization of BTs with the ternary logic $K_3$.
        \item
            We introduce mixed-integer encodings for STL and formalized BTs over the ternary logic $K_3$.
        \item
            We demonstrate the utility of the formalized BT encodings for control by solving linear discrete-time optimal control problems.
    \end{enumerate}

    This paper is organized as follows: Section~\ref{sec:related-work} describes work related to the formalization of BTs including our previous work on the subject. In Section~\ref{sec:preliminaries}, we provide the necessary preliminaries on ternary logic, temporal logic, and system modeling. We formulate the main problem in Section~\ref{sec:problem-formulation} and present our technical solution in Section~\ref{sec:encodings}. Section~\ref{sec:experiments} demonstrates our framework with two optimal control tasks.

    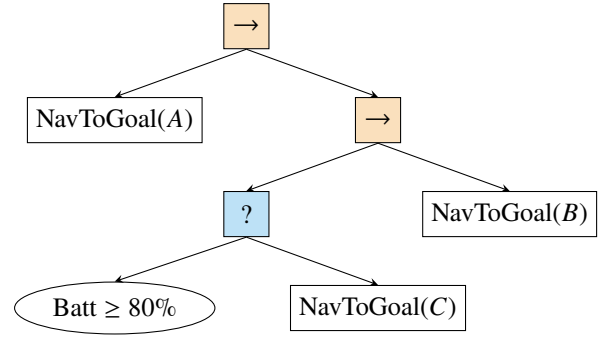
\begin{figure}[tp]
        \centering
        % New TBT diagram
%
\begin{tikzpicture}
    \path   (0, 0)         node[draw, fill=goldencheese!70, minimum width=.6cm, minimum height=.6cm]           (root)   {$\rightarrow$}
           +(-1.75, -1.25) node[draw, minimum width=.6cm, minimum height=.6cm]           (EA)     {NavToGoal($A$)}
          ++(1.75, -1.25)  node[draw, fill=goldencheese!70, minimum width=.6cm, minimum height=.6cm]           (seq)    {$\rightarrow$}
           +(1.75, -1.25)  node[draw, minimum width=.6cm, minimum height=.6cm]           (EB)     {NavToGoal($B$)}
          ++(-1.75, -1.25) node[draw, fill=Cerulean!25, minimum width=.6cm, minimum height=.6cm]           (sel)    {?}
           +(-1.75, -1.25) node[draw, ellipse, minimum width=.75cm, minimum height=.6cm] (batt)   {$\text{Batt}\geq80\%$}
           +(1.75, -1.25)  node[draw, minimum width=.6cm, minimum height=.6cm]           (charge) {NavToGoal($C$)};

    \draw[-stealth] (root.south)--(EA.north);
    \draw[-stealth] (root.south)--(seq.north);
    \draw[-stealth] (seq.south)--(sel.north);
    \draw[-stealth] (sel.south)--(batt.north);
    \draw[-stealth] (sel.south)--(charge.north);
    \draw[-stealth] (seq.south)--(EB.north);
\end{tikzpicture}
        \caption{
            Example BT for a mobile robot task. BTs are read from left to right with the colored nodes representing the BT control operators and the leaf nodes representing execution nodes. This BT specifies that the mobile robot must first navigate to goal $A$. Then, if the battery is greater than or equal to $80\%$, proceed to goal $B$. Otherwise, if the battery is less than $80\%$, navigate to goal $C$ then proceed to goal $B$.}
        \label{fig:TBT}
    \end{figure}
\section{Related Work\label{sec:related-work}}
    Task and motion planning (TAMP) using temporal logic has a long history covering many varieties of temporal logic and system models. TAMP with Linear Temporal Logic (LTL) specifications~\cite{pnueli1977temporal} considers automata-theoretic techniques, where a continuous system is abstracted to a transition system and a satisfying plan is constructed for it. Each state in the transition system then maps to a real-valued feedback control strategy that, when implemented, satisfies the top-level LTL formula~\cite{kloetzer2006fully,fainekos2009temporal}. TAMP with Metric Interval Temporal Logic (MITL) specifications~\cite{alur1996benefits} considers automata-theoretic, optimization-based, and model-free (reinforcement learning) control techniques covering a wide variety of autonomous systems including robotic manipulators and UAVs~\cite{baras2015MTL,baras2019monMTL,baras2020RLMITL}. Real-valued logics like STL have been widely investigated for control synthesis. Techniques include mixed-integer programming, gradient-based methods using the real-valued robustness semantics, and model-free control with reinforcement learning. These techniques cover a large variety of autonomous systems including linear, nonlinear, and stochastic models~\cite{belta2019formal,kurtz2022mixed,raman2014model,wolff2014optimization,baras2020plmonSTL,haghighi2019control,kordabad2024control}.
    
    The authors in~\cite{schirmer2024temporal} introduce the notion of utilizing BTs as task specifications through Temporal BTs (TBTs), which is an offline temporal logic monitor. This extends STL to include BT-inspired operators. A TBT formula is equipped with qualitative and quantitative semantics that indicate the degree of satisfaction (robustness) of a trajectory with respect to the TBT formula, and is computed using dynamic programming. In~\cite{schirmer2025trace}, the authors introduce mixed-integer encodings for TBTs using the Boolean qualitative semantics for trace repair, computing a satisfying trace by minimally perturbing a violating trace. As we concretize later in this paper, BTs operate in a ternary logic domain, and therefore require a ternary logic system to accurately formalize them. 
    
    In prior works, we have investigated the formalization of BTs over ternary logic for control synthesis. In~\cite{matheu2025bt2automata}, we show an isomorphism between formalized BTs and finite state automata, allowing synthesis for discrete transition systems. In~\cite{matheu2025omtbt}, we develop an online TBT robustness monitor and synthesize control policies for unmodeled systems via reinforcement learning. To the best of our knowledge, these are the only works that use a ternary logic system and perform control synthesis with respect to TBT formulas.
\section{Preliminaries\label{sec:preliminaries}}

    \subsection{Behavior Trees}

        BTs for controlling autonomous systems~\cite{colledanchise2018behavior} are directed acyclic graphs consisting of execution and control nodes (see Fig. \ref{fig:TBT} for an example). Leaf nodes are \textit{execution} nodes, responsible for observing or controlling the state of the system. Connecting the execution nodes, and responsible for the overall structure of the tree, are \textit{control} nodes. The control nodes determine the execution flow of the tree and return conditions based on the status of their children. Each node in a BT has three possible conditions: \textit{Success}, \textit{Failure}, and \textit{Running}. The control node types are:
    
        \noindent\textbf{Sequence Node (${\seq}/{\rightarrow}$):} The \textit{Sequence} node can have $k\geq1$ children. The Sequence node executes children from left-to-right contingent on each child returning \textit{Success}. The Sequence node returns \textit{Success} if all children succeed, \textit{Failure} if any child fails, and \textit{Running} otherwise.
    
        \noindent\textbf{Selector Node (${\sel}/{?}$):} The \textit{Selector} node can have $k\geq1$ children. The Selector is the dual of the Sequence node. The Selector node executes children from left-to-right contingent on each child returning \textit{Failure}. The Selector node returns \textit{Failure} if all children fail, \textit{Success} if any child succeeds, and \textit{Running} otherwise.
        
        Leaf nodes are classified into two types: \textit{action} and \textit{condition} nodes. Action nodes control the state of the system, returning a logical status of \textit{Running} while executing. Upon termination, they must return a status of \textit{Success} or \textit{Failure}. Action nodes are represented graphically with a rectangle. Condition nodes query the state of the system and instantly return a logical status, either \textit{Success} or \textit{Failure}. Condition nodes are represented graphically with an ellipse.
    
    \subsection{Signal Temporal Logic}
        We consider the discrete-time domain as the set of all nonnegative integers: $\mathbb{T}:=\{t\in\mathbb{Z}: t\geq0\}$. The state of a dynamical system is represented by a real vector $x_t\in\mathbb{R}^n$ for time step $t\in\mathbb{T}$. A system trajectory is an infinite tuple of states $x=(x_0,\,x_1,\,x_2,\dots)$. Because we are working with discrete time indices, a time interval is formally defined as the set $[a,\,b]:=\{t\in\mathbb{T}: a\leq t\leq b,\,a,b\in\mathbb{T}\}$. Note that this definition includes time intervals that contain a single time step, as is the case when $a=b$.
        
        STL~\cite{maler2004monitoring} provides a formalism for specifying spatiotemporal behaviors of real-valued systems. The semantics of STL are defined over infinite trajectories, where the Boolean satisfaction of an STL formula is interpreted as either \textit{True} or \textit{False}. The basic primitive in STL is a \textit{signal predicate} of the form $\mu:=f(x_t)\geq0$ where $f:\mathbb{R}^n\to\mathbb{R}$ is a Lipschitz continuous function. An STL formula is defined using the following grammar:
        \[
            \varphi:=\mu\,\vert\,\neg\varphi\,\vert\,\varphi_1\wedge\varphi_2\,\vert\,\square_{[a,\,b]}\varphi\,\vert\,\lozenge_{[a,\,b]}\varphi,
        \]
        where $[a,\,b]$ is a time interval. Formulas are recursively defined starting with the signal predicates and composed together using logical connectives (${\neg},\,{\wedge}$) and temporal operators. Other familiar logical connectives such as Boolean \textit{or} (${\vee}$), \textit{implication} (${\rightarrow}$), and \textit{equivalence} (${\leftrightarrow}$) are defined using De Morgan's laws. $\square_{[a,\,b]}$ is the unary \textit{always} operator and $\lozenge_{[a,\,b]}$ is the unary \textit{eventually} operator. As an example, a reach-and-hold specification can be formulated as $\lozenge_{[0,\,10]}\square_{[0,\,2]}(y\geq 5)$, meaning $y$ must become greater than or equal to $5$ within 11 time steps and remain greater than or equal to $5$ for 3 time steps (here $t=0$ counts as the first time step). The \textit{horizon} of an STL formula is the number of time steps for which future values are needed to compute the satisfaction of the formula at the current time step. A recursive definition for the horizon of an STL formula is provided in~\cite{dokhanchi2014line}. For instance, the previous reach-and-hold specification requires $11+3=14$ time steps \textit{at most} to compute satisfaction.

        STL admits two semantics: the qualitative (Boolean) semantics, and the quantitative (robustness) semantics~\cite{donze2010robust}. The robustness semantics provide a real-valued metric indicating how much a trajectory satisfies or violates a formula. However, for the purposes of mixed-integer programming, we will only consider the qualitative semantics.

        \begin{definition}[STL Qualitative Semantics \cite{maler2004monitoring}]
            \label{def:stl-qual}
            The qualitative semantics for STL map real-valued trajectories to an interpretation of a formula, and are defined inductively as follows:
            \[
                \begin{aligned}
                    x,\,t\models\mu\quad&\leftrightarrow\quad \mu(x_t), \\
                    x,\,t\models\neg\varphi\quad&\leftrightarrow\quad\neg(x,\,t\models\varphi), \\
                    x,\,t\models\varphi_1\wedge\varphi_2\quad&\leftrightarrow\quad(x,\,t\models\varphi_1)\wedge(x,\,t\models\varphi_2), \\
                    x,\,t\models\square_{[a,\,b]}\varphi\quad&\leftrightarrow\quad\forall\tau\in[a+t,\,b+t]:x,\,\tau\models\varphi, \\
                    x,\,t\models\lozenge_{[a,\,b]}\varphi\quad&\leftrightarrow\quad\exists\tau\in[a+t,\,b+t]:x,\,\tau\models\varphi,
                \end{aligned}
            \]
            where the statement $x,\,t\models\varphi$ means trajectory $x$ satisfies (models) formula $\varphi$ at time step $t$. The notations $\neg(x,\,t\models\varphi)$ and $x,\,t\not\models\varphi$ are used interchangeably.
        \end{definition}

    \subsection{Temporal Behavior Trees\label{subsec:TBT}}
        TBTs~\cite{schirmer2024temporal} extend STL by introducing new $k$-ary temporal operators inspired by the BT control nodes. The syntax of TBTs is defined using the following grammar:
        \[
            \varphi:=\psi\,\vert\,\seq{(\varphi_1,\dots,\,\varphi_k)}\,\vert\,\sel{(\varphi_1,\dots,\,\varphi_k)},
        \]
        where $\psi$ is an STL formula. TBT formulas build on top of STL formulas meaning all STL formulas are valid TBTs. TBTs admit Boolean qualitative and quantitative semantics, the details of which are provided in~\cite{schirmer2024temporal}. Informally, the qualitative semantics of the $\seq$ operator state that each subformula $\varphi_1$ through $\varphi_k$ must be satisfied one after the other in sequence. The $\sel$ operator must satisfy one subformula, where each subformula $\varphi_1$ through $\varphi_k$ is evaluated sequentially. We revisit the formal definition of the qualitative semantics in Section~\ref{sec:encodings} where we define them over a ternary logic.

        TBTs not only provide the ability to specify complex sequences of behaviors, but also have an interpretable graph structure. For instance, the TBT formula $\seq{(\lozenge A,\,\seq{(\sel{(\text{Batt}\geq80\%,\,\lozenge C)},\,\lozenge B)})}$ has the equivalent graphical representation shown in Fig.~\ref{fig:TBT}. To be consistent with the BT notion of action and condition nodes, we restrict condition nodes to untimed formulas such that they evaluate in a single time step. Action nodes are more general and can consist of timed formulas.

    \subsection{Kleene's Strong Logic}
        Kleene's strong logic ($K_3$)~\cite{kleene1952introduction} proposes a third truth value \textit{Unknown} (sometimes referred to as \textit{Undetermined}) to represent indeterminacy in computation, concretizing the case when a proposition is neither \textit{True} nor \textit{False}.

        \begin{definition}[$K_3$ Syntax and Semantics]
            \label{def:k3}
            Let $P$ be a set of propositional variables that take values in the set $\{False,\,Unknown,\,True\}$ (abbr. $\{F,\,U,\,T\}$). A formula in $K_3$ is defined inductively as: 1) every $p\in P$ is a formula; 2) if $\varphi$ is a formula, then $\neg \varphi$ is a formula; 3) if $\varphi_1$ and $\varphi_2$ are formulas, then $\varphi_1\wedge \varphi_2$ and $\varphi_1\vee \varphi_2$ are formulas. The interpretations of the connectives $\neg$, $\wedge$, and $\vee$ are defined as:
            
            \begin{center}
                \begin{tabular}{c|c}
                    $\neg$ & \\
                    \hline
                    $F$ & $T$ \\
                    $U$ & $U$ \\
                    $T$ & $F$
                \end{tabular}
                \quad
                \begin{tabular}{c|ccc}
                    $\wedge$ & $F$ & $U$ & $T$ \\
                    \hline
                    $F$ & $F$ & $F$ & $F$ \\
                    $U$ & $F$ & $U$ & $U$ \\
                    $T$ & $F$ & $U$ & $T$
                \end{tabular}
                \quad
                \begin{tabular}{c|ccc}
                    $\vee$ & $F$ & $U$ & $T$ \\
                    \hline
                    $F$ & $F$ & $U$ & $T$ \\
                    $U$ & $U$ & $U$ & $T$ \\
                    $T$ & $T$ & $T$ & $T$
                \end{tabular}
            \end{center}
        \end{definition}
        Let $\mathbb{K}=\{F,\,U,\,T\}$ be the set of ternary truth constants. $K_3=(\mathbb{K},\,\wedge,\,\vee,\,\neg,\,F,\,T)$ forms a Kleene algebra~\cite{kaarli2000polynomial} with $F<U<T$. Negation is an involution satisfying $\neg(x\wedge y)\,\leftrightarrow\,\neg x\vee\neg y$ for all $x,\,y\in\mathbb{K}$. Logical implication retains its usual definition: $x\rightarrow y:=\neg x\vee y$. Where $K_3$ differs from traditional Boolean algebras is in its failure to satisfy the complement laws, i.e., $x\wedge\neg x= F$ and $x\vee\neg x= T$. These laws do not hold in the case $x=U$.
        
        As we have demonstrated in previous work~\cite{matheu2025bt2automata, matheu2025omtbt}, formalization of BTs requires a three-valued logic due to the \textit{Running} state that BT nodes may assume during execution. The instantaneous state of a BT node can be \textit{Failure} (equiv. \textit{False}, $F$), \textit{Running} (equiv. \textit{Unknown}, $U$), or \textit{Success} (equiv. \textit{True}, $T$). The logic system $K_3$ provides a concrete formalism for modeling the logical state of BTs.
\section{Problem Formulation\label{sec:problem-formulation}}
   
    A linear discrete-time dynamical system is of the form
    \begin{equation}
        \label{eq:sys-dyn}
        \begin{aligned}
            x_{t+1}&=A_tx_t+B_tu_t \\
            y_t&=C_tx_t+D_tu_t
        \end{aligned}
    \end{equation}
    where $x_t\in\mathbb{R}^n$ is the state vector at time step $t\in\mathbb{T}$, $u_t\in \mathcal{U}\subseteq\mathbb{R}^m$ is the control vector at time step $t$, and $y_t\in\mathbb{R}^p$ is the output at time step $t$. The matrices $A_t\in\mathbb{R}^{n\times n}$, $B_t\in\mathbb{R}^{n\times m}$, $C_t\in\mathbb{R}^{p\times n}$, and $D_t\in\mathbb{R}^{p\times m}$ may depend on time alone. A linear discrete-time optimal control problem with temporal logic constraints seeks to minimize a convex objective function subject to the dynamical constraints given in Eq.~\eqref{eq:sys-dyn}, $u_t\in\mathcal{U}$, and additional constraints given as temporal logic specifications on the system trajectory. Approaches to solving optimal control problems of this type are traditionally done in one of two ways: 1) Formulating the problem as a mixed-integer program using constraints from the temporal logic's qualitative semantics. This approach is guaranteed to find a globally optimal solution, but suffers from high complexity. 2) Gradient-based methods using a differentiable relaxation of the quantitative semantics as a maximum objective. This approach is typically much quicker to return a solution, but may get stuck in a local optimum.

    In this paper, we consider the optimal control of linear dynamical systems of the form~\eqref{eq:sys-dyn}, where constraints on the system trajectory are provided as a TBT formula. Following the conventions of temporal logic constrained control synthesis, a feasible trajectory with respect to a TBT formula is guaranteed by construction to satisfy both safety and behavioral requirements defined in the formal specification.
    
    The semantics of temporal logic formulas are conventionally defined with respect to infinite trajectories, leading to undecidability when synthesizing finite trajectories. To circumvent this issue, formulas are chosen such that the horizon of the formula is at most $N$, allowing the formula to be decided given a finite system trajectory. Defining the qualitative semantics of STL over $K_3$ allows for a third logical state \textit{Unknown}, allowing formulas whose horizons exceed $N$ to still be decided. Therefore the semantics of STL over $K_3$ generalize to trajectories of any length. Solving optimal control problems with TBT constraints requires proper representations for the qualitative semantics of STL over $K_3$. In our approach, we formulate ternary mixed-integer encodings where integer decision variables are now \textit{trits}, taking values in the set $\{{-}1,\,0,\,{+}1\}$. In addition to encodings for STL, we develop ternary encodings for the TBT Sequence and Selector operators, thus allowing for control with respect to formulas defined over the TBT grammar in Section~\ref{subsec:TBT}.

    To make use of efficient mixed-integer optimization tools, we restrict the objective function to be quadratic in the control effort, and integer decision variables constraining the trajectory of the system may appear only in the constraints of the optimization problem. We formulate the optimal control problem as the following mathematical program:
    \begin{prob}
        \label{prob:the-prob}
        Given a system of the form~\eqref{eq:sys-dyn}, a quadratic cost function, and a TBT formula $\varphi$ constraining the trajectory of the system, solve:
        \begin{equation}
            \begin{aligned}
                u^*=\arg&\min_{u}\sum_{t=0}^{N-1}{u_t^\top Ru_t} \\
                \text{\textnormal{s.t.}}\qquad&x_{t+1}=A_tx_t+B_tu_t,\,t=0,\dots,\,N-1 \\
                &u_t\in\mathcal{U},\,t=0,\dots,\,N-1 \\
                &x_0=\xi \\
                &x,\,t^*\models\varphi
            \end{aligned}
        \end{equation}
        where $R$ is a positive semidefinite weight matrix, $\mathcal{U}\subseteq\mathbb{R}^m$ is a convex polytope, $\xi$ is known, and $x,\,t^*\models\varphi$ indicates the synthesized trajectory satisfies a TBT formula $\varphi$ beginning at a user-defined time step $t^*$.
    \end{prob}
    Note that the above program is a Mixed-Integer Quadratic Program (MIQP) due to the quadratic objective and the constraint $x,\,t^*\models\varphi$, which is to be realized with mixed-integer linear constraints.
\section{Ternary Encodings for Behavior Tree Operators\label{sec:encodings}}
    \subsection{STL Over Partial Trajectories}

        We now extend the qualitative semantics of STL to the ternary logic $K_3$, allowing formulas to be decided in the case of partial system trajectories. A partial trajectory is a contiguous subsequence of an infinite system trajectory. Partial trajectories have a final time step, which we refer to as the partial trajectory horizon, or partial horizon. A partial trajectory ending at partial horizon $t_2$ is a finite tuple $x_\text{partial}=(x_{0},\,x_{1},\dots,\,x_{t_2-1},\,x_{t_2})$ where $t_2\geq0$, i.e., a partial trajectory always contains at least one point. Given a partial trajectory and an STL formula, it may be the case that there is not enough information to definitively determine the satisfaction of the formula. In this case we interpret the satisfaction of the formula to be \textit{Unknown}, thus requiring a redefinition of the STL operators over $K_3$. Reasoning over $K_3$ has the added benefit of capturing uncertainty in system measurements. This is implemented by allowing signal predicates to evaluate to \textit{Unknown} whenever $f(x_t)$ falls in a particular interval. We devise the more general ternary signal predicate of the form:
        \begin{equation}
            \label{eq:tern-pred}
            \mu(x_t):=\begin{cases}
                \;T, & f(x_t)\geq\delta, \\
                \;U, & -\delta<f(x_t)<\delta \\
                \;F, & f(x_t)\leq-\delta,
            \end{cases}
        \end{equation}
        where $\delta>0$ is a user-defined uncertainty threshold. 
        \begin{remark}
        If we want $\delta=0$, we revert to the classical signal predicate formulation in which $\mu=T\,\rightarrow\,f(x_t)\geq0$ and $\mu=F\,\rightarrow\,f(x_t)<0$. 
        \end{remark}

        Let the statement $x,\,t_1,\,t_2\models\varphi$ be the satisfaction of $\varphi$ at time step $t_1$ over trajectory $x$ with partial horizon $t_2$. The satisfaction at time step $t_1\leq t_2$ takes values in the ternary set of truth constants: $(x,\,t_1,\,t_2\models\varphi)\in\{F,\,U,\,T\}$. We now formulate the qualitative semantics of partial trajectory STL by extending the Boolean semantics in Def.~\ref{def:stl-qual}. A similar formulation for online STL semantics is demonstrated in~\cite{bellanger2025formally}.
        Given a trajectory $x$ and a partial horizon $t_2$, the satisfaction of an STL formula at time step $t_1\leq t_2$ is defined inductively starting with the untimed connectives:
        \[
            \begin{aligned}
                x,\,t_1,\,t_2\models\mu\;&\leftrightarrow\;\mu(x_{t_1}), \\
                x,\,t_1,\,t_2\models\neg\varphi\;&\leftrightarrow\;\neg(x,\,t_1,\,t_2\models\varphi),\\
                x,\,t_1,\,t_2\models\varphi_1\wedge\varphi_2\;&\leftrightarrow\;(x,\,t_1,\,t_2\models\varphi_1)\wedge(x,\,t_1,\,t_2\models\varphi_2),
            \end{aligned}
        \]
        where the logical connectives adhere to the rules in Def.~\ref{def:k3}. The notations $\neg(x,\,t_1,\,t_2\models\varphi)$ and $x,\,t_1,\,t_2\not\models\varphi$ are used interchangeably. Let the \textit{search window} of a timed operator be the time interval in which the operand is evaluated. For example, the statement $x,\,t_1,\,t_2\models\square_{[a,\,b]}\varphi$ is determined by evaluating $\varphi$ over the interval $[a+t_1,\,b+t_1]$, where $[a+t_1,\,b+t_1]$ is the search window. The timed operators elicit branching behavior dependent on the relationship between the search window and the partial horizon. The satisfaction of $x,\,t_1,\,t_2\models\square_{[a,\,b]}\varphi$ is equivalently:
        \[
            \begin{cases}
                \;\bigwedge_{\tau=a+t_1}^{b+t_1}{(x,\,\tau,\,t_2\models\varphi)},&b+t_1\leq t_2, \\
                \;F,\qquad\exists \tau\in[a+t_1,\,t_2]: x,\,\tau,\,t_2\not\models\varphi\quad\wedge&b+t_1>t_2, \\
                \;U, &\text{otherwise},
            \end{cases}
        \]
        and $x,\,t_1,\,t_2\models\lozenge_{[a,\,b]}\varphi$ is equivalently:
        \[
            \begin{cases}
                \;\bigvee_{\tau=a+t_1}^{b+t_1}{(x,\,\tau,\,t_2\models\varphi)}, &b+t_1\leq t_2, \\
                \;T,\qquad\exists\tau\in[a+t_1,\,t_2]:x,\,\tau,\,t_2\models\varphi\quad\wedge &b+t_1>t_2, \\
                \;U, &\text{otherwise}.
            \end{cases}
        \]

        In the case that the partial horizon is greater than or equal to the search window for the operator, the semantics reduce to the familiar infinite trajectory semantics. In the case where the search window for the operator exceeds the partial horizon, the satisfaction remains \textit{Unknown} unless there is a short-circuit. A short-circuit occurs when an operand causes a conjunction or disjunction to invariably evaluate to a constant. For instance, a conjunction will short-circuit to \textit{False} if any conjunct is \textit{False} and a disjunction will short-circuit to \textit{True} if any disjunct is \textit{True}. Note that in the case $a+t_1>t_2$, the disjunction $\exists\tau\in[a+t_1,\,t_2]$ is vacuously \textit{False} and the satisfaction evaluates to \textit{Unknown} by default for both operators.
    
        Typical mixed-integer encoding schemas for STL formulas like those found in~\cite{belta2019formal, kurtz2022mixed, raman2014model, wolff2014optimization} deal with the Boolean semantics for temporal logics where \textit{True} is encoded as the integer $1$ and \textit{False} is encoded as the integer $0$. The satisfaction of a formula $\varphi$ at a time step $t$ is encoded as a binary decision variable $z_t^\varphi\in\{0,\,1\}$. To enforce the satisfaction of a top-level formula at a particular time step $t^*$, the additional constraint $z_{t^*}^\varphi=1$ is added. Although reasoning over a finite trajectory, it is assumed that the horizon of $\varphi$ is less than the length of the finite trajectory, and therefore the constraint $z_{t^*}^\varphi=1$ is feasible.
    
        To encode the partial trajectory STL semantics over $K_3$, we choose the encoding schema: $True\mapsto {+}1$, $Unknown\mapsto 0$, and $False\mapsto {-}1$, thus preserving the order $F<U<T$. Decision variables are now ternary integers where the satisfaction of a formula $\varphi$ at time step $t$ takes values $z_t^\varphi\in\{{-}1,\,0,\,{+}1\}$. Let $N\in\mathbb{T}$ be the final time step for a finite horizon optimal control problem. Introduce $N+1$ continuous decision variables $x_t\in\mathbb{R}^n$ for $t=0,\dots,\,N$ representing the state trajectory. Next introduce $N$ continuous decision variables $u_t\in\mathcal{U}$ for $t=0,\dots,\,N-1$ representing the control input. The system dynamics are encoded as $N$ linear constraints by declaring $x_{t+1}=A_tx_t+B_tu_t$ for $t=0,\dots,\, N-1$ and $x_0=\xi$.
        
        We begin with the ternary predicate encoding. We restrict predicate functions to be affine in the state: $f(x_t):=a^Tx_t-b$ where $a\in\mathbb{R}^n$ and $b\in\mathbb{R}$. Let $z_t^\mu\in\{{-}1,\,0,\,{+}1\}$ represent the ternary valuation of a signal predicate $\mu$ of the form~\eqref{eq:tern-pred} at time step $t$. For each time step, introduce three binary decision variables: $w_t^{{+}1}$, $w_t^{0}$, and $w_t^{{-}1}$. We constrain the binary decision variables such that only one is active at any given time and reconstruct our ternary decision variable:
        \[
            w_t^{{+}1}+w_t^{0}+w_t^{{-}1}=1,\qquad z_t^\mu=w_t^{{+}1}-w_t^{{-}1}.
        \]
        The above constraints ensure that $w_t^{k}=1\,\rightarrow\,z_t^\mu=k$ for $k\in\{{-}1,\,0,\,{+}1\}$. Let $M$ be a large constant such that $M\geq\vert a^\top x_t-b\vert$ for all $t$. Enforcing the valuation of the ternary predicate requires four linear constraints:
        \[
            \begin{aligned}
                a^\top x_t-b&\geq\delta -M(1-w_t^{{+}1}), \\
                a^\top x_t-b&\leq{-}\delta+M(1-w_t^{{-}1}), \\
                a^\top x_t-b&\leq\delta-\varepsilon+M(1-w_t^{0}), \\
                a^\top x_t-b&\geq{-}\delta+\varepsilon-M(1-w_t^{0}),
            \end{aligned}
        \]
        where $\varepsilon>0$ is a small numerical constant that enforces a strict inequality in the form of a nonstrict inequality.
        \begin{remark}
            The case where we want $\delta=0$ requires only two binary decision variables: $w_t^{{+}1}$ and $w_t^{{-}1}$ where $w_t^{{+}1}+w_t^{{-}1}=1$, $z_t^\mu=w_t^{{+}1}-w_t^{{-}1}$, and two linear constraints:
            \[
                \begin{aligned}
                    a^\top x_t-b&\geq\delta -M(1-w_t^{{+}1}), \\
                    a^\top x_t-b&\leq{-}\delta+\varepsilon+M(1-w_t^{{-}1}).
                \end{aligned}
            \]
        \end{remark}
        Negation is trivial: introduce a ternary decision variable $z_t^{\neg\varphi}$ and the constraint $z_t^{\neg\varphi}={-}z_t^{\varphi}$. Conjunctive and disjunctive constraints for $z_j\in\{{-}1,\,0,\,{+}1\}$ are of the form: $z^\wedge=\bigwedge_{j=1}^{m}z_j\,\rightarrow\, z^\wedge=\min_{j}{z_j}$, and $z^\vee=\bigvee_{j=1}^{m}z_j\,\rightarrow\,z^\vee=\max_{j}{z_j}$.
        Conjunctive constraints are realized by the following schema: 1) Enforce the upper bound $z^\wedge\leq z_j$. 2) Introduce $m$ binary variables $b_j$ such that $\sum_{j=1}^{m}b_j=1$ (i.e., only one is active). 3) Then the lower bound for $z^\wedge$ is determined by $z^\wedge\geq z_j-2(1-b_j)$. Disjunctive constraints are realized by the following schema: 1) Enforce the lower bound $z^\vee\geq z_j$. 2) Again introduce $m$ binary decision variables $b_j$ such that $\sum_{j=1}^{m}b_j=1$. 3) Then the upper bound is enforced by the constraints $z^\vee\leq z_j+2(1-b_j)$.

        The partial trajectory temporal operators always and eventually introduce a second time dimension in the form of the partial horizon. Let $\varphi$ be an untimed formula, containing predicates and logical connectives only. The satisfaction of a timed formula at time step $t_1$ with partial horizon $t_2$ is encoded as a ternary decision variable by the following schema: If $t_2\geq b+t_1$, then the encodings are simply:
        \begin{equation}
            z_{t_1,\,t_2}^{\square_{[a,b]}\varphi}=\bigwedge_{j=a+t_1}^{b+t_1}{z_j^\varphi},\qquad z_{t_1,\,t_2}^{\lozenge_{[a,b]}\varphi}=\bigvee_{j=a+t_1}^{b+t_1}{z_j^\varphi}.
        \end{equation}
        In any other case, $b+t_1>t_2$ and the partial trajectory can evaluate to \textit{Unknown} or short-circuit. Therefore the encodings become:
        \begin{equation}
            z_{t_1,\,t_2}^{\square_{[a,b]}\varphi}=U\wedge\bigwedge_{j=a+t_1}^{t_2}z_{j}^\varphi,\qquad z_{t_1,\,t_2}^{\lozenge_{[a,b]}\varphi}=U\vee\bigvee_{j=a+t_1}^{t_2}z_{j}^\varphi,
        \end{equation}
        where the \textit{Unknown} constant acts as a logical bound, ensuring that the partial trajectory can evaluate to at most \textit{Unknown} for the always operator or at least \textit{Unknown} for the eventually operator.

        In general, the encodings for all partial trajectory STL formulas (timed or untimed) have two temporal dimensions: the partial horizon ($t_2$) and the time step in which the formula is being evaluated ($t_1$). The partial horizon is always greater than or equal to the time step in which the formula is being evaluated, as any case where $t_1>t_2$ vacuously evaluates to \textit{Unknown}. In the case of untimed formulas, the decision variables reduce to one dimension with $z_{t_1,\,t_2}^\varphi=z_{t_1}^\varphi$. When encoding nested temporal operators, the partial horizon remains consistent. For example, given decision variables $z_{t_1,\,t_2}^\varphi$ representing the partial trajectory satisfaction of a timed formula $\varphi$, let $\psi=\square_{[a,\,b]}\varphi$. Then the encoding for $\psi$ at time $t_1$ with $b+t_1\leq t_2$ is: $z_{t_1,\,t_2}^\psi=\bigwedge_{j=a+t_1}^{b+t_1}z_{j,\,t_2}^\varphi$, where the partial horizon $t_2$ is the same for all decision variables.

    \subsection{Sequence Operator}
        By now we have a full encoding schema for partial trajectory STL over the $K_3$ ternary logic, and we turn our attention to the TBT operators. The following schema is derived from semantics introduced in~\cite{schirmer2024temporal} and Boolean mixed-integer encodings in~\cite{schirmer2025trace}. We begin with the Sequence operator which admits $k\geq 1$ subformulas. The satisfaction of a Sequence operator with one subformula is trivial: $x,\,t_1,\,t_2\models\seq{(\varphi)}\,\leftrightarrow\,x,\,t_1,\,t_2\models\varphi$. A Sequence operator with two subformulas has the following semantics: $x,\,t_1,\,t_2\models\seq{(\varphi_1,\,\varphi_2)}$ iff:
        \[
            \exists\tau\in[t_1,\,t_2-1]:(x,\,t_1,\,\tau\models\varphi_1)\,\wedge\,(x,\,\tau+1,\,t_2\models\varphi_2),
        \]
        meaning there exists a partial trajectory satisfying $\varphi_1$ and a partial trajectory beginning at the next time step satisfying $\varphi_2$.
        \begin{lemma}
            The Sequence operator is associative. That is: if $x,\,t_1,\,t_2\models\seq{(\seq{(\varphi_1,\,\varphi_2)},\,\varphi_3)}$, then $x,\,t_1,\,t_2\models\seq{(\varphi_1,\,\seq{(\varphi_2,\,\varphi_3)})}$. The converse is also true.
        \end{lemma}
        \begin{proof}
            We prove logical equivalence.

            ($\rightarrow$): Assume $x,\,t_1,\,t_2\models\seq{(\seq{(\varphi_1,\,\varphi_2)},\,\varphi_3)}$. Then there exists $\tau\in[t_1,\,t_2-1]$ such that $x,\,t_1,\,\tau\models\seq{(\varphi_1,\,\varphi_2)}$ and $x,\,\tau+1,\,t_2\models\varphi_3$. Expanding the inner Sequence reveals that there exists $s\in[t_1,\,\tau-1]$ such that $x,\,t_1,\,s\models\varphi_1$ and $x,\,s+1,\,\tau\models\varphi_2$. Then $x,\,s+1,\,\tau\models\varphi_2$ and $x,\,\tau+1,\,t_2\models\varphi_3$ implies $x,\,s+1,\,t_2\models\seq{(\varphi_2,\,\varphi_3)}$ which, along with $x,\,t_1,\,s\models\varphi_1$, implies $x,\,t_1,\,t_2\models\seq{(\varphi_1,\,\seq{(\varphi_2,\,\varphi_3)})}$.

            ($\leftarrow$): Assume $x,\,t_1,\,t_2\models\seq{(\varphi_1,\,\seq{(\varphi_2,\,\varphi_3)})}$. Then there exists $\tau\in[t_1,\,t_2-1]$ such that $x,\,t_1,\,\tau\models\varphi_1$ and $x,\,\tau+1,\,t_2\models\seq{(\varphi_2,\,\varphi_3)}$. Expanding the inner Sequence reveals that there exists $s\in[\tau+1,\,t_2-1]$ such that $x,\,\tau+1,\,s\models\varphi_2$ and $x,\,s+1,\,t_2\models\varphi_3$. Then $x,\,t_1,\,\tau\models\varphi_1$ and $x,\,\tau+1,\,s\models\varphi_2$ implies $x,\,t_1,\,s\models\seq{(\varphi_1,\,\varphi_2)}$. This, along with $x,\,s+1,\,t_2\models\varphi_3$, implies that $x,\,t_1,\,t_2\models\seq{(\seq{(\varphi_1,\,\varphi_2)},\,\varphi_3)}$.
        \end{proof}
        \begin{corl}
        The associative property of the Sequence operator is particularly nice because it allows us to write $k$-ary formulas in the following recursive form:
        \begin{equation}
            \seq{(\varphi_1,\dots,\,\varphi_k)}=\seq{(\varphi_1,\,\varphi_{S_{k-2}})},
        \end{equation}
        where $\varphi_{S_m}=\seq{(\varphi_{k-m},\,\varphi_{S_{m-1}})}$ for $m=2,\dots,\,k-2$ and $\varphi_{S_1}=\seq{(\varphi_{k-1},\,\varphi_k)}$.
        \end{corl}
        
         Let $z_{t_1,\,t_2}^\varphi$ be a ternary decision variable representing the satisfaction of a Sequence formula $\varphi=\seq{(\varphi_1,\,\varphi_2)}$. The integer encodings for $\varphi$ are given by:
        \begin{equation}
            z_{t_1,\,t_2}^\varphi=\bigvee_{\tau=t_1}^{t_2-1}\left(
                z_{t_1,\,\tau}^{\varphi_1}\,\wedge\,z_{\tau+1,\,t_2}^{\varphi_2}
            \right).
        \end{equation}

    \subsection{Selector Operator}
        The Selector operator is the logical dual of the Sequence operator. The satisfaction of a Selector operator with one subformula is trivial: $x,\,t_1,\,t_2\models\sel{(\varphi)}\,\leftrightarrow\,x,\,t_1,\,t_2\models\varphi$. A Selector operator with two subformulas has the following semantics: $x,\,t_1,\,t_2\models\sel{(\varphi_1,\,\varphi_2)}$ iff:
        \[
            \exists\tau\in[t_1,\,t_2-1]:(x,\,t_1,\,\tau\models\varphi_1)\vee(x,\,\tau+1,\,t_2\models\varphi_2).
        \]
        In other words, the Selector operator is satisfied if a partial trajectory satisfies $\varphi_1$ or the partial trajectory beginning at the next time step satisfies $\varphi_2$. Like its dual, the Selector operator is associative, allowing us to write $k$-ary formulas in the following recursive form: $\sel{(\varphi_1,\dots,\,\varphi_k)}=\sel{(\varphi_1,\,\varphi_{S_{k-2}})}$, where $\varphi_{S_m}=\sel{(\varphi_{k-m},\,\varphi_{S_{m-1}})}$ for $m=2,\dots,\,k-2$ and $\varphi_{S_1}=\sel{(\varphi_{k-1},\,\varphi_k)}$. Let $z_{t_1,\,t_2}^\varphi$ be a ternary decision variable representing the satisfaction of a Selector formula $\varphi=\sel{(\varphi_1,\,\varphi_2)}$. The integer encodings for $\varphi$ are given by:
        \begin{equation}
            z_{t_1,\,t_2}^\varphi=\bigvee_{\tau=t_1}^{t_2-1}\left(
                z_{t_1,\,\tau}^{\varphi_1}\,\vee\,z_{\tau+1,\,t_2}^{\varphi_2}
            \right).
        \end{equation}

        Given encodings for a TBT formula $\varphi$ in the form of decision variables $z_{t_1,\,t_2}^\varphi$, satisfaction at a time step $t^*$ is enforced by encoding that $\varphi$ is satisfied over any partial horizon: $\exists\tau\in[t^*+1,\,N]: z_{t^*,\,\tau}^\varphi=1$, where $N$ is the final time step. Alternatively, we can reduce the number of constraints if we know $\varphi$ is the top-level specification by relaxing the disjunction and enforcing instead $z_{t^*,\,N}^\varphi=1$.

        \begin{thm}
            Ternary encodings of the Sequence operator with $k$ subformulas require $O\big({(t_2-t_1)}^{k-1}\big)$ constraints.
        \end{thm}
        \begin{proof}
            By the recursive property of the Sequence operator, the general $k$-ary encoding has the following form:
            \[
                \bigvee_{\tau_1=t_1}^{t_2-1}\bigvee_{\tau_2=\tau_1+1}^{t_2-1}\cdots\bigvee_{\tau_{k-1}=\tau_{k-2}+1}^{t_2-1}\left(
                    z^{\varphi_1}_{t_1,\tau_1}\wedge z^{\varphi_2}_{\tau_1+1,\tau_2}\wedge\cdots\wedge z^{\varphi_k}_{\tau_{k-1}+1,t_2}
                \right)
            \]
            which is a disjunction of $\binom{t_2-t_1}{k-1}$ conjuncts, where each conjunct contains $k$ ternary decision variables.
        \end{proof}

    \subsection{Complexity}
        Let $N\in\mathbb{T}$ be the final time step for an optimal control problem and $\{\mu_1,\dots,\,\mu_k\}$ be a set of ternary predicates. Each ternary predicate requires three binary decision variables and four linear constraints for each time step $t\in[0,\,N]$. Encoding all predicates brings the total number of binary decision variables to $(N+1)\times k\times 3=O(N\times k)$ and the total number of linear constraints to $(N+1)\times k\times 4=O(N\times k)$. Conjunctive and disjunctive constraints of the form $\varphi_1\wedge\varphi_2$ and $\varphi_1\vee\varphi_2$ where $\varphi_1$ and $\varphi_2$ are untimed formulas add $2\times(N+1)=O(N)$ binary decision variables and $4\times(N+1)=O(N)$ linear constraints.

        Timed formulas require two temporal dimensions, $t_1,\,t_2\in[0,\,N]$ with $t_2\geq t_1$, forming an upper triangular matrix of dimension $(N+1)\times(N+1)$. Thus there are a total of $\frac{(N+2)\times(N+1)}{2}=O(N^2)$ nontrivial decision variables per timed formula. Encodings for formulas of the form $\psi=\square_{[a,\,b]}\varphi$ require at most $b-a$ conjunctive constraints per time step in both dimensions. Therefore fully spanning a formula over both temporal dimensions requires $O(N^2)$ conjunctive constraints and $O(N^2)$ binary variables.
        
        Encoding Sequence and Selector formulas requires $O\big({(t_2-t_1)}^{k-1}\big)$ conjunctive or disjunctive constraints per time pair $(t_1,\,t_2)$, where $k$ is the number of subformulas. Over the entire horizon, the number of constraints behaves like $O(N^{k+1})$. This is a stark degradation in complexity compared to traditional STL encodings, which are linear in both the number of constraints and variables due to there being only one temporal dimension. Furthermore, the complexity of MILP/MIQP problems grows exponentially in the worst case with respect to the number of binary decision variables~\cite{belta2019formal}.
\section{Case Studies\label{sec:experiments}}

    TBT specifications are often much richer than traditional STL specifications. Consider the following formulas that constrain the temporal sequencing of two subtasks: $\psi_1=\lozenge_{[0,\,\tau]}\varphi_1\wedge\lozenge_{[\tau+1,\,N]}\varphi_2$ and $\psi_2=\seq{(\lozenge_{[0,\,N]}\varphi_1,\,\lozenge_{[0,\,N]}\varphi_2)}$. Both formulas constrain the behavior of the system such that $\varphi_1$ is completed before $\varphi_2$ within $N$ time steps. However, $\psi_1$ has the additional hyperparameter $\tau$ which temporally segments satisfaction of $\varphi_1$ from satisfaction of $\varphi_2$. $\psi_2$ places no such constraint on the timing of the sequencing, only that $\varphi_1$ must be completed at some time before completion of $\varphi_2$. We note that $\psi_2$ is a richer specification than $\psi_1$ in that any trajectory that satisfies $\psi_1$ also satisfies $\psi_2$. The additional behaviors are a natural consequence of the ternary logic encodings where behaviors that were otherwise unrealizable are now quantified by evaluating to \textit{Unknown}, allowing for a much larger search space of possible solutions. We elucidate our results with two case studies demonstrating control from TBT specifications that are not simply represented with STL.

    \subsection{Discrete-Time Double Integrator}
        Consider a mobile robot with discrete-time dynamics of the form~\eqref{eq:sys-dyn} where:
        \begin{equation}
            \label{eq:dint}
            A=\begin{bmatrix}
                \;1 & \delta t & 0 & 0\; \\
                \;0 & 1 & 0 & 0\; \\
                \;0 & 0 & 1 & \delta t\; \\
                \;0 & 0 & 0 & 1\;
            \end{bmatrix},\quad B=\begin{bmatrix}
                \;{\delta t}^2/2 & 0\; \\
                \;\delta t & 0\; \\
                \;0 & {\delta t}^2/2\; \\
                \;0 & \delta t\;
            \end{bmatrix},
        \end{equation}
        $x_t=(p_t^1,\,v_t^1,\,p_t^2,\,v_t^2)\in\mathbb{R}^4$ is the Cartesian position and velocity, $u_t=(u_t^1,\,u_t^2)\in\mathcal{U}\subseteq\mathbb{R}^2$, and $\delta t>0$ is the fixed sampling interval. The system is initialized at $x_0=(0,\,0,\,0,\,0)$ indicating the robot starts at position $(0,\,0)$ and is initially at rest. The robot is constrained to satisfy a TBT specification derived from the BT in Fig.~\ref{fig:TBT}, with leaf nodes formalized as $\varphi_A:=\lozenge_{[0,\,N]}(x_t\in A)$, $\varphi_B:=\lozenge_{[0,\,N]}(x_t\in B)$, and $\varphi_C:=\lozenge_{[0,\,N]}(x_t\in C)$. The goal locations $A$, $B$, and $C$ are boxes in $\mathbb{R}^2$ with uncertainty threshold $\delta=0.25$ to encourage the robot to enter the interior of the box. The top-level formula is then $\varphi\wedge\square(\neg O_1\wedge \neg O_2)$ where $O_1$ and $O_2$ are obstacles and $\varphi$ is the TBT formula:
        \begin{equation}\label{eqn:phi1}
        \varphi=\seq{(\varphi_A,\,\seq{(\sel{(\text{Batt}\geq80\%,\,\varphi_C)},\,\varphi_B)})}.
        \end{equation}
        Here Batt is an additional state variable representing the initial charge of the robot's battery. Depending on the value of Batt, the system will experience a branching behavior when evaluating the Selector operator. To encode obstacle avoidance, we first represent the obstacles $O_1$ and $O_2$ as polygons via finite intersections of halfspaces. Each halfspace becomes a ternary predicate of the form $z_t^j=1\,\rightarrow\,a_j^\top(p_t^1,\,p_t^2)\geq b_j+\delta$ for all $j$ facets in the polygon where $a_j\in S^1$ and $\delta>0$ is a minimum distance threshold to ensure that the robot is not touching the facet. Obstacle avoidance is simply the disjunction of all facet predicates, i.e., $\neg O_1\,\leftrightarrow\,\bigvee_j z_j$.
        
        For this experiment we set the final time step to $N=25$, with a time interval $\delta t=0.5$s. Our objective seeks to minimize the control effort $J(u)=\sum_{t=0}^{N-1}u_t^\top u_t$, and we limit the control effort such that $u_t^1,\,u_t^2\in[-1,\,1]$ for all $t$. The model has 148 continuous variables, 4800 integer variables, 4775 linear constraints, and 48 quadratic objective terms. The model was solved to global optimality using Gurobi~\cite{gurobi} with an academic license. The solution was found in 97.28 seconds on an Ubuntu 22.04 machine with an AMD Ryzen 9 7950X 16 Core CPU and 128GB of memory. The synthesized trajectories for two different initial configurations of the battery level are shown in Fig.~\ref{fig:ex1}.
        \begin{figure}[htbp]
            \centering
            \subfigure[$\text{Batt}\geq80\%$]{\includegraphics[width=.498\linewidth]{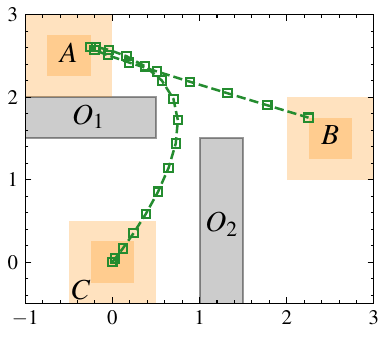}}
            \hfill
            \subfigure[$\text{Batt}<80\%$]{\includegraphics[width=.492\linewidth]{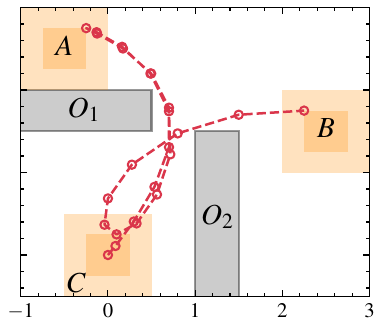}}
            \caption{Synthesized Robot Trajectories satisfying the TBT formula~\eqref{eqn:phi1} for two different initial battery levels. When the battery level is below $80\%$, the robot must return to the charger (location $C$) before proceeding to location $B$. In both cases the robot starts at the center of location $C$ and is initially at rest.}
            \label{fig:ex1}
        \end{figure}

    \subsection{Multi-Agent Planning}
        We consider a multi-agent mobile robot system with each agent's dynamics modeled as a discrete-time double integrator. The overall dynamics for the multi-agent system with $M$ agents are modeled via the system matrices: $A=\text{BlockDiag}(A_1,\dots,\,A_M)$ and $B=\text{BlockDiag}(B_1,\dots,\,B_M)$, where each $A_i$ and $B_i$ is from Eq.~\eqref{eq:dint}. The synthesized trajectories are to satisfy the top-level TBT specification:
        \begin{equation}
            \label{eq:ex2-spec}
            \begin{split}
                \varphi=\seq&{(\varphi_1,\,\varphi_2)}\,\wedge\,\square(\neg O_1\wedge\neg O_2) \\
                &\wedge\,\square(d((p^{1,i}_t,\,p^{2,i}_t),\,(p^{1,j}_t,\,p^{2,j}_t))\geq d_{\min}).
            \end{split}
        \end{equation}
        Here $\varphi_1$ and $\varphi_2$ are sub-level tasks specifying that each agent must eventually reach a particular goal location. The second conjunct specifies that the agents must not intersect the obstacle regions $O_1$ and $O_2$. The last conjunct specifies that the agents must always keep a minimum distance between themselves and the other agents. We choose $d(\cdot,\,\cdot)$ to be the Manhattan distance in order to preserve linearity of the encodings with $d_{\min}=0.6$. We restrict the control effort for each agent such that $u_t^{1,j},\,u_t^{2,j}\in[-1,\,1]$ for all $t$ and seek to minimize the total control effort $J(u)=\sum_{t=0}^{N-1}u_t^\top u_t$. The synthesized trajectories for $3$ agents are shown in Fig.~\ref{fig:ex2} along with the details for both sub-level tasks.
        \begin{figure}[htbp]
            \centering
            \includegraphics[width=\linewidth]{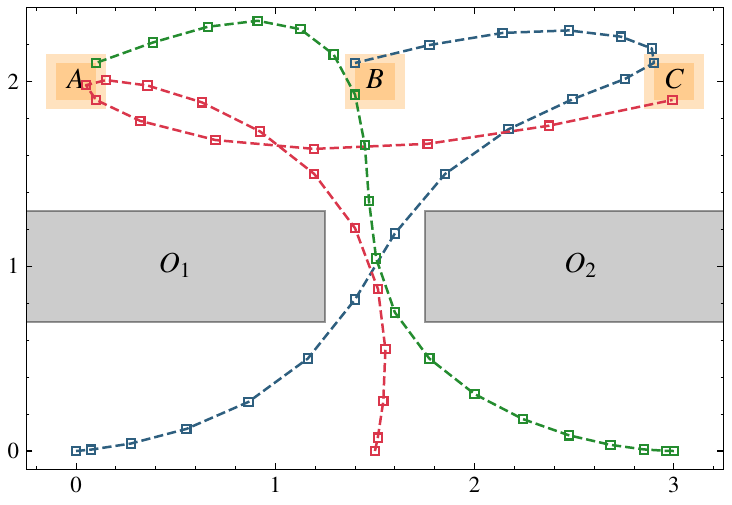}
            \caption{Synthesized trajectories for three agents satisfying the TBT formula~\eqref{eq:ex2-spec} with maximum finite horizon $N=20$ time steps and time interval $\delta t=0.5$s. Sub-level formula $\varphi_1$ specifies that eventually agent 1 ({\color{exblue}blue}) shall reach goal $C$, agent 2 ({\color{exred}red}) shall reach goal $A$, and agent 3 ({\color{exgreen}green}) shall reach goal $B$. Sub-level formula $\varphi_2$ specifies that eventually agent 1 shall reach goal $B$, agent 2 shall reach goal $C$, and agent 3 shall reach goal $A$. The model has 354 continuous variables, 2040 integer variables, 2040 linear constraints, and 114 quadratic objective terms. The model was solved to global optimality in 307.75 seconds on an Ubuntu 22.04 machine with an AMD Ryzen 9 7950X 16 Core CPU and 128GB of memory.}
            \label{fig:ex2}
        \end{figure}
        An interesting result arising from constraining a minimum distance between agents is that the agents take turns navigating through the obstacle corridor. This queuing phenomenon is a natural consequence of strict safety specifications and is highlighted in Fig.~\ref{fig:queue}. 
        \begin{figure}[htbp]
            \centering
            \includegraphics[width=\linewidth]{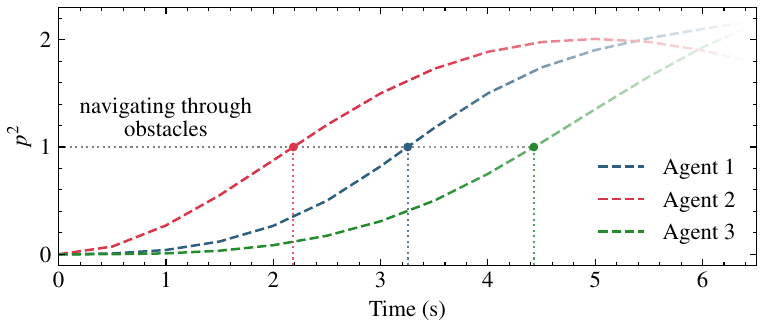}
            \caption{The synthesized multi-agent trajectories elicit a queuing phenomenon when passing through the obstacles.}
            \label{fig:queue}
        \end{figure}
\section{Conclusion\label{sec:conclusion}}
    This work is a continuation of our investigation into the logical formalization of BTs, specifically for application to control synthesis. We concretize the semantics of TBTs over the ternary logic $K_3$ and develop mixed-integer linear encodings for TBTs over ternary STL. Ternary logic is the natural choice for BTs due to the fact that BTs operate in a three-state domain. Defining the qualitative semantics for STL over $K_3$ has the added benefit of capturing uncertainty, allowing signals to evaluate to \textit{Unknown} whenever a predicate is neither definitively \textit{True} nor \textit{False}. Our encoding schema allows for verification tasks and control synthesis via mixed-integer programming, highlighted through two optimal control examples. As part of future work, we will investigate verification and control synthesis with respect to formalized BTs for nonlinear systems, further generalizing the class of systems to which our methods are applicable.
\bibliographystyle{mybibstyle}
\bibliography{refs}
\end{document}